\documentclass[11pt,letterpaper]{article}
\usepackage[top=1in, bottom=1in, left=1in, right=1in]{geometry}

\usepackage[utf8]{inputenc} 
\usepackage[T1]{fontenc}    
\usepackage{hyperref}       
\usepackage{url}            
\usepackage{booktabs}       
\usepackage{amsfonts}       
\usepackage{nicefrac}       
\usepackage{microtype}      
\usepackage{xcolor}         
\usepackage{lmodern}

\usepackage{subcaption}
\usepackage[font=small,labelfont=bf]{caption}
\usepackage{multicol}
\usepackage{graphicx}

\RequirePackage{fancyhdr}
\RequirePackage{xcolor} 
\RequirePackage{algorithm}
\RequirePackage{algorithmic}
\RequirePackage{natbib}
\RequirePackage{eso-pic} 
\RequirePackage{forloop}
\RequirePackage{url}

\usepackage{amsmath, amsthm, amssymb}
\usepackage{bm}
\usepackage{color}
\usepackage{ulem}
\usepackage{mathtools}

\let\emph\textit

\DeclarePairedDelimiterX{\infdivx}[2]{(}{)}{%
    #1\;\delimsize\|\;#2%
}

\theoremstyle{plain}
\newtheorem{theorem}{Theorem}[section]
\newtheorem{proposition}[theorem]{Proposition}
\newtheorem{lemma}[theorem]{Lemma}
\newtheorem{corollary}[theorem]{Corollary}
\theoremstyle{definition}
\newtheorem{definition}[theorem]{Definition}

\theoremstyle{remark}
\newtheorem{remark}[theorem]{Remark}

\newcommand{\const}[1]{\mathrm{#1}}
\def\R{{\mathbb{R}}}

\def\1{\textbf{1}}
\def\econst{\const{e}}

\newcommand{\onevct}{\bm{1}}

\newtheorem{problem}{Problem}

\newcommand{\Prob}[2][]{\mathbb{P}_{#1}\left\{ {#2} \right\}}
\newcommand{\Expect}[2][]{\mathbb{E}_{#1}\left[ #2 \right]}

\newcommand{\abs}[1]{\left\vert {#1} \right\vert}

\newcommand{\norm}[1]{\left\Vert {#1} \right\Vert}

\newcommand{\normf}[1]{{\norm{#1}}_\text{F}}

\newcommand{\norminf}[1]{{\norm{#1}}_{\infty}}

\newcommand{\vect}[1]{\operatorname*{vec}\left({#1}\right)}

\newcommand{\st}{\operatorname*{subject\; to}}
\newcommand{\maximize}{\operatorname*{maximize}}
\newcommand{\minimize}{\operatorname*{minimize}}

\newcommand{\lmax}{\operatorname{\lambda_{\max}}}
\newcommand{\lmin}{\operatorname{\lambda_{\min}}}
\newcommand{\lsec}{\lambda_{2}}

\newcommand{\inprod}[2][]{\left\langle {#1},{#2} \right\rangle}

\newcommand{\om}{{\otimes m}}

\newcommand{\yast}{{y^{\ast}}}

\newcommand{\Vcal}{\mathcal{V}}
\newcommand{\Ecal}{\mathcal{E}}
\newcommand{\Hcal}{\mathcal{H}}
\newcommand{\Gcal}{\mathcal{G}}
\newcommand{\Ical}{\mathcal{I}}

\newcommand{\cpsd}{\mathcal{S}_{+}^{n,m}}
\newcommand{\vpsd}{\mathcal{S}_{+}^{\ast,n,m}}

\newcommand{\sgmm}{{\sigma_2^{n,m}}}
\newcommand{\sgmc}{{\bar{\sigma}_2^{n,m}}}

\usepackage{bbm}
\newcommand{\onefun}{\mathbbm{1}}

\makeatletter
\def\algbackskip{\hskip-\ALG@thistlm}
\makeatother

\title{Exact Inference in High-order Structured Prediction}

\author{
  \textbf{Chuyang Ke}\\Department of Computer Science\\Purdue University\\\texttt{cke@purdue.edu}
  \and 
  \textbf{Jean Honorio}\\Department of Computer Science\\Purdue University\\\texttt{jhonorio@purdue.edu}
}

\date{}
\begin{document}
\maketitle

\begin{abstract}
In this paper, we study the problem of inference in high-order structured prediction tasks. 
In the context of Markov random fields, the goal of a high-order inference task is to maximize a score function on the space of labels, and the score function can be decomposed into sum of unary and high-order potentials. 
We apply a generative model approach to study the problem of high-order inference, and provide a two-stage convex optimization algorithm for exact label recovery. 
We also provide a new class of hypergraph structural properties related to hyperedge expansion that drives the success in general high-order inference problems.
Finally, we connect the performance of our algorithm and the hyperedge expansion property using a novel hypergraph Cheeger-type inequality. 
\end{abstract}

\allowdisplaybreaks

\section{Introduction}

\emph{Structured prediction} has been widely used in various machine learning fields in the past $20$ years, including applications like social network analysis, computer vision, molecular biology, natural language processing (NLP), among others. 
A common objective in these tasks is assigning / recovering labels, that is, given some possibly noisy observation, the goal is to output a group label for each entity in the task. 
In social network analysis, this could be detecting communities based on user profiles and preferences \citep{kelley2012defining}.
In computer vision, researchers want the AI to decide whether a pixel is in the foreground or background \citep{nowozin2011structured}. 
In biology, it is sometimes desirable to cluster molecules by structural similarity \citep{nugent2010overview}.
In NLP, part-of-speech tagging is probably one of the most well-known structured prediction task \citep{weiss2010structured}. 

From a methodological point of view, a standard approach in the structured prediction tasks above, is to recover the \emph{global structure} by exploiting many \emph{local structures}. 
Take social networks as an example. A widely used assumption in social network analysis is \emph{affinity} --- users with similar profiles and preferences are more likely to become friends. Intuitively, a structured prediction algorithm tends to assign two users the same label, if they have a higher affinity score. 
Similarly, the same idea can be motivated in the context of Markov random fields (MRFs). Assume all entities form an undirected graph $\Gcal = (\Vcal, \Ecal)$, structured prediction can be viewed as the task of solving the following inference problem \citep{bello2019exact}:
\begin{align}
\maximize_{y \in \mathcal{L}^{\abs{\mathcal{V}}}}
\sum_{v\in \Vcal, l\in\mathcal{L}} c_v(l) \cdot \onefun[y_v = l] 
+ \sum_{\substack{(v_1,v_2)\in \Ecal \\ l_1, l_2\in\mathcal{L}}} c_{v_1,v_2}(l_1,l_2) \cdot \onefun[y_{v_1} = l_1, y_{v_2} = l_2] 
\,,
\label{eq:mrf_graph_opt}
\end{align} 
where $\mathcal{L}$ is the space of labels, $c_v(l)$ is the score of assigning label $l$ to node $v$, and $c_{v_1,v_2}(l_1,l_2)$ is the score of assigning labels $l_1$ and $l_2$ to neighboring nodes $v_1$ and $v_2$. In the MRF and inference literature, the two terms in \eqref{eq:mrf_graph_opt} are often referred to as \emph{unary} and \emph{pairwise potentials}, respectively. 
The inference formulation above allows one to recover the global structure, by finding a configuration that maximizes the summation of unary and pairwise local scores. 

However, entities in many real-world problems could interact beyond the pairwise fashion. Take the social network example again, but this time let us focus on the academia co-authorship network: many published papers are written by more than two authors \citep{liu2005co}. Such high-order interactions cannot be captured by pairwise structures.
Geometrically, the co-authorship network can no longer be represented by a graph. As a result, the introduction of hypergraphs is necessary to model high-order structured prediction problems.

In this paper, we study the problem of high-order structured prediction, in which instead of using pairwise potentials, \emph{high-order potentials} are considered. Using the MRF formulation, we are interested in inference problems of the following form:
\begin{align}
\maximize_{y \in \mathcal{L}^{\abs{\mathcal{V}}}} 
\sum_{v\in \Vcal, l\in\mathcal{L}} c_v(l) \cdot \onefun[y_v = l] 
+ \sum_{\substack{e\in \Ecal \\ l_1, \dots, l_m\in\mathcal{L} \\ e = (v_1,\dots,v_m)}} c_{e}(l_1,\dots,l_m) \cdot \onefun[y_{v_1} = l_1, \dots, y_{v_m} = l_m] 
\,,
\label{eq:mrf_hypergraph_opt}
\end{align} 
where $m$ is the order of the inference problem as well as the hypergraph (each hyperedge connects $m$ vertices), and $c_{e}(l_1,\dots,l_m)$ is the score of assigning labels $l_1$ through $l_m$ to neighboring nodes $v_1$ through $v_m$ connected by hyperedge $e \in \Ecal$. 



\subsection{Inference as a Recovery Task}

Structured prediction and inference problems with unary and pairwise potentials in the form of \eqref{eq:mrf_graph_opt} have been studied in prior literature. 
\citet{globerson2015hard} introduced the problem of label recovery in the case of two-dimensional grid lattices, and analyzed the conditions for approximate inference.
Along the same line of work, \citet{foster2018inference} generalized the model by allowing tree decompositions. 
Another flavor is the problem of exact inference, for which \citet{bello2019exact} proposed a convex semidefinite programming (SDP) approach. 
In these works, the problem of label recovery is motivated by a generative model, which assumes the existence of a ground truth label vector $\yast$, and generates (possibly noisy) unary and pairwise observations based on label interactions.

Unfortunately, little is known for structured prediction with high-order potentials and hypergraphs. 
In recent years, there have been various attempts to generalize some graph properties (including hypergraph Laplacian, Rayleigh quotient, hyperedge expansion, Cheeger constant, among others) to hypergraphs \citep{li2018submodular, mulas2021cheeger, yoshida2019cheeger,chan2018spectral, chen2017fiedler,chang2020hypergraph}.
However, it took a long time for us to find out that due to the nature of structured prediction tasks, the hypergraph definitions must fulfill certain properties, and none of the definitions in the aforementioned works fulfill those. This makes it challenging to design hypergraph-based label recovery algorithms. Furthermore, none of the aforementioned works provide guarantees of either approximate or exact inference.

In this work, we apply a generative model approach to study the problem of high-order inference. We analyze the task of label recovery, and answer the following question:
\begin{problem}[Label Recovery]
Is there any algorithm that takes noisy unary and high-order local observation as the input, and correctly recovers the underlying true labels? 
\end{problem}


\subsection{Inference and Structural Properties}
In the MRF inference literature, a central and longtime discussion focuses on inference solvability versus certain structural properties of the problem. 
 
\begin{figure*}[ht!]
\centering

\begin{subfigure}[b]{.3\linewidth}
\centering
\includegraphics[width=\linewidth]{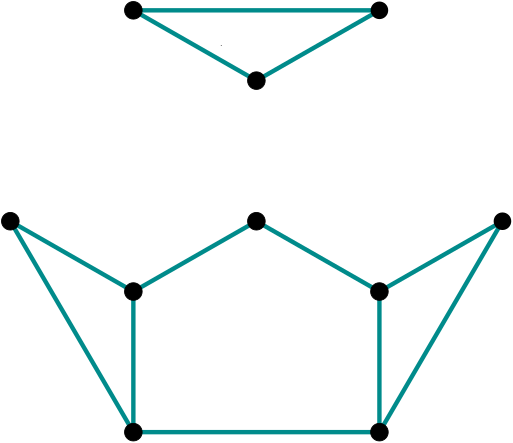}
\caption{Disconnected graph (zero edge expansion)}
\label{fig:link}
\end{subfigure}
\hspace{3mm}
\begin{subfigure}[b]{.3\linewidth}
\includegraphics[width=\linewidth]{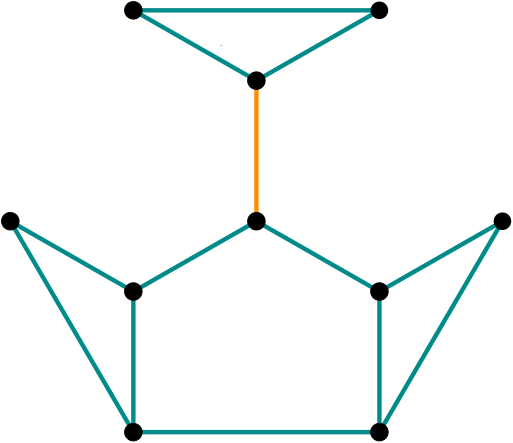}
\caption{Graph with a bottleneck (small edge expansion)}
\label{fig:star}
\end{subfigure}
\hspace{3mm}
\begin{subfigure}[b]{.3\linewidth}
\includegraphics[width=\linewidth]{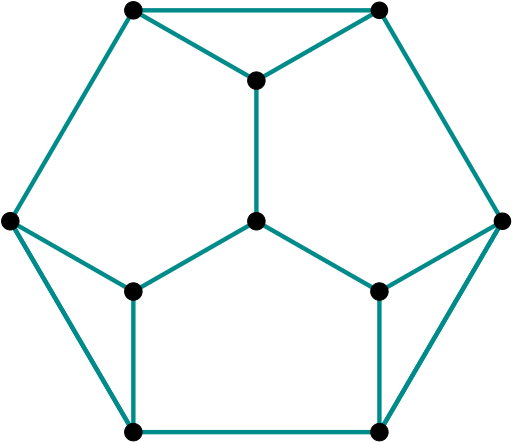}
\caption{Graph without bottlenecks (large edge expansion)}
\label{fig:star2}
\end{subfigure}

\caption{Graph expansion examples. Figure \ref{fig:link} shows a disconnected graph. With an ideal algorithm we may be able to recover the user communities in the top subgraph and in the bottom subgraph, but since there is no observed interaction between the two, we will not be able to infer the global structure.  
Figure \ref{fig:star} connects the two components in Figure \ref{fig:link} using a single orange edge. In this case the orange edge is the  ``bottleneck.'' Removing the orange edge disconnects the graph.
Figure \ref{fig:star2} adds two more edges to Figure \ref{fig:star}. In this case every component is connected with no weak ``bottleneck.'' 
Edge expansion is a structural property, which characterizes how connected the components in a graph are.}
\label{figs:edge_expansion}
\end{figure*}

To see this, we first revisit various classes of structural properties in pairwise inference problems (i.e., in the form of \eqref{eq:mrf_graph_opt}). 
\citet{chandrasekaran2012complexity} studied treewidth in graphs as a structural property, and showed that graphs with low treewidths are solvable. 
\citet{schraudolph2008efficient} demonstrated that planar graphs can be solved.
\citet{boykov2006graph} analyzed graphs with binary labels and sub-modular pairwise potentials. 
\citet{bello2019exact} showed that inference with graphs that are ``good'' expanders, or ``bad'' expanders plus an Erdos-Renyi random graph, can be achieved.
It is worth highlighting that the structural properties above are not directly comparable or reducible. Instead, they characterize the difficulty of an inference problem from different angles, or in other words, for different classes of graphs.

Similar discussions about inference versus structural properties exist in high-order MRF inference literature.
For example, \citet{komodakis2009beyond} investigated hypergraphs fulfilling the property of one sub-hypergraph per clique, and proved that high-order MRF inference can be achieved through solving linear programs (LPs).
\citet{fix2014hypergraph} studied inference in hypergraphs with the property of local completeness.
\citet{gallagher2011inference} analyzed the performance of high-order MRF inference through order reduction versus the number of non-submodular edges. 
However, these works do not provide theoretical guarantees of either approximate or exact inference.

In this paper, we provide a new class of hypergraph structural properties by analyzing \emph{hyperedge expansion}. In order to get some intuition, let us consider a social network with two disconnected sub-networks. With an ideal algorithm, we may recover the user communities in subnet 1 and the those in subnet 2, but since there is no interaction between the two subnets at all, we will not be able to infer the global community structure (e.g., the relationship between the recovered communities in subnet 1 and subnet 2). 
A less extreme case is networks with ``bottlenecks,'' i.e., removing these bottleneck edges will disconnect the network. For similar reasons, one can imagine that inference in networks with bottlenecks can be hard if noise is present.
See Figure \ref{figs:edge_expansion} for an illustration.
In pairwise graphs ($2$-graphs), such connectivity / bottleneck property can be characterized by the edge expansion (i.e., the Cheeger constant) of the graph.
Characterizing similar expansion properties in high-order hypergraphs poses a challenge, especially if one wants to relate such topological properties to the conditions of exact inference.

\begin{problem}[Structural Property]
Under what topological conditions will our label recovery algorithm work correctly with high probability?
\end{problem}

\textbf{Summary of our contribution.} 
Our work is highly theoretical. We provide a series of novel definitions and results in this paper:
\begin{itemize}
    \item We provide a new class of hypergraph structural properties for high-order inference problems. We derive a novel Cheeger-type inequality, which relates the tensor spectral gap of a hypergraph Laplacian to a Cheeger-type hypergraph expansion property. These hypergraph results are not only limited to the scope of the model in this paper, but also can be helpful to researchers working on high-order inference problems.
    \item We propose a two-stage approach to solve the problem of high-order structured prediction. We formulate the label recovery problem as a high-order combinatorial optimization problem, and further relax it to a novel convex conic form optimization problem. 
    \item We carefully analyze the Karush–Kuhn–Tucker (KKT) conditions of the conic form optimization problem, and derive the sufficient statistical and topological conditions for exact inference. Our KKT analysis guarantees the solution to be optimal with a high probability, as long as the conditions are fulfilled.  
\end{itemize}

\section{Preliminaries}

In this section, we formally define the high-order exact inference problem and introduce the notations that will be used throughout the paper.

We use lowercase font (e.g., $a,b,u,v$) to denote scalars and vectors, and uppercase font (e.g., $A,B,C$) to denote tensors. 
We denote the set of real numbers by $\R$.

For any natural number $n$, we use $[n]$ to denote the set $\{1,\dots, n\}$.

We use $\onevct$ to denote the all-ones vector.

For clarity we use superscripts ${(i)}$ to denote the $i$-th object in a sequence of objects, and subscripts $j$ to denote the $j$-th entry.
We use $\circ$ to denote the Hadamard product, and $\otimes$ to denote the outer product. 
Let $v^{(1)},\dots,v^{(m)}\in \R^n$ be a sequence of $m$ vectors of dimension $n$. Then $v^{(1)} \otimes \ldots \otimes v^{(m)}$ is a tensor of order $m$ and dimension $n$ (or equivalently, of shape $n^\om$), such that 
$
(v^{(1)} \otimes \ldots \otimes v^{(m)})_{i_1, \ldots, i_m} = v^{(1)}_{i_1} \ldots v^{(m)}_{i_m}.
$


\subsection{Tensor Definitions}

Let $A \in \R^{n^\om}$ be an $m$-th order, $n$-dimensional real tensor. Throughout the paper, we limit our discussion to $m = 2,6,10,14,\dots$ for clarity of exposition. While other even orders ($m=4, 8, 12\dots$) are possible and a similar analysis will follow, the hypergraph definitions will be involving many more terms and the paper will be less readable. See Remark \ref{remark:m_order} for discussion.

A symmetric tensor is invariant under any permutation of the indices. In other words, $A$ is symmetric if for any permutation $\sigma:[m]\to [m]$, we have $A_{\sigma(i_1,\dots,i_m)} = A_{i_1,\dots,i_m}$.

We define the inner product of two tensors $A$, $B$ of the same shape as  $\inprod[A]{B} := \sum_{i_1,\ldots,i_m = 1}^{n} A_{i_1 ,\ldots, i_m} B_{i_1 ,\ldots, i_m}$. We define the tensor Frobenius norm as $\normf{A} := \sqrt{\inprod[A]{A}}$.

A symmetric tensor $A$ is positive semidefinite (PSD), if for all $v\in \R^n$, we have $\inprod[A]{v^\om} \geq 0$. We use $\cpsd$ to denote the convex cone of all $m$-order, $n$-dimensional PSD tensors. 

The dual cone of $\cpsd$ is the Caratheodory tensor cone $\vpsd$, which is defined as $\vpsd := \left\{\sum_{i=1}^{\binom{m+n-1}{m}} v^{(i)\om} \mid v^{(i)} \in \R^n \right\}$. In other words, every tensor in $\vpsd$ is the summation of at most $\binom{m+n-1}{m}$ rank-one tensors. 
$\cpsd$ and $\vpsd$ are dual to each other \citep{ke2022exact}.

For any tensor $A \in \R^{n^\om}$, we define its minimum tensor eigenvalue $\lmin(A)$ (or equivalently $\lambda_1(A)$) using a variational characterization, such that 
$\lmin(A) := \min_{v\in\R^n, \norm{v} = 1} \inprod[A]{v^\om}$.
Similarly we define its second minimum tensor eigenvalue $\lsec(A)$ as
$\lsec(A) := \min_{v\in\R^n, \norm{v} = 1, v\perp v^{(1)}} \inprod[A]{v^\om}$, where $v^{(1)}$ is the eigenvector corresponding to $\lmin(A)$.

We denote $\sgmm$ as the index set of $m$-tuples in the shape of $\sigma(i_1,i_1,i_2,i_2,\dots,i_{m/2}, i_{m/2})$, for any permutation $\sigma:[m] \to [m]$ and $i_j \in [n]$. Intuitively, in every tuple of $\sgmm$, every index repeats an even number of times. We use $\sgmc$ to denote the set $\{(i_1,\dots,i_m) \mid (i_1,\dots,i_m) \notin \sgmm, \text{and at least one index repeats twice}\}$. In other words, $\sgmc$ is the complement of $\sgmm$, subtracting cases with all unique indices.


\subsection{High-order Inference Task}


We consider the task of predicting a set of $n$ vertex labels $\yast = (y_1^\ast, \dots, y_n^\ast)$, where $y_i^\ast \in \{+1, -1\}$, from a set of observations $X$ and $z$. 
$X$ and $z$ are noisy observations generated from some underlying $m$-uniform hypergraph $\Gcal = (\Vcal, \Ecal)$. 
In particular, $\Vcal$ is the set of vertices (nodes) with $\abs{\Vcal} = n$,
and $\Ecal$ is the set of hyperedges.

For every possible $m$-vertex tuple $e = ({i_1},\dots,{i_m})$, if $e \in \Ecal$, the hyperedge observation $X_{{i_1},\dots,{i_m}}$ (and all corresponding symmetric entries $X_{\sigma({i_1},\dots,{i_m})}$) is sampled to be $y_1^\ast \cdots y_m^\ast$ with probability $1-p$, and $- y_1^\ast \cdots y_m^\ast$ with probability $p$ independently. If $e \notin \Ecal$, $X_{{i_1},\dots,{i_m}}$ is set to $0$.

For every node $v_i \in \Vcal$, the node observation $z_i$ is sampled to be $y_i^\ast$ with probability $1-q$, and $- y_i^\ast$ with probability $q$ independently.

We now summarize the generative model.

\begin{definition}[High-order Structured Prediction with Partial Observation]
\textbf{Unknown}: True node labeling vector $\yast = (y_1^\ast,\dots, y_n^\ast)$.
\textbf{Observation}: Partial and noisy hyperedge observation tensor $X \in \{-1,0,+1\}^{n^\om}$. 
Noisy node label observation vector $z \in \{-1,+1\}^n$.
\textbf{Task}: Infer and recover the correct node labeling vector $\yast$ from the observation $X$ and $z$.
\end{definition}
\section{Hypergraph Structural Properties and Cheeger-type Inequality}
In this section, we introduce a series of novel Cheeger-type analysis for hypergraphs. This allows us to characterize the spectral gap of a hypergraph Laplacian, via the topological hyperedge expansion of the graph itself.
Hypergraph theorems in this section are general, and are not limited to the specific model covered in our inference task. To the best of our knowledge, the following high-order definitions and results are novel. 
All missing proofs of the lemmas and theorems can be found in Appendix \ref{appendix:proof_lemma}.


\subsection{Hypergraph Topology}

We first introduce the necessary hypergraph topological definitions. 

\begin{definition}[Induced Hypervertices]
Given an $m$-uniform hypergraph $\Gcal = (\Vcal, \Ecal)$, we use 
\[
\Hcal := \{\{i_1,\dots,i_{m/2}\} \mid i_1, \dots, i_{m/2} \in [n]\}
\] 
to denote the set of induced hypervertices. 
We denote its cardinality by $N := \abs{\Hcal} = \binom{n}{m/2}$.
\label{def:hypervertices}
\end{definition}

\begin{definition}[Boundary of a Hypervertex Set]
For any hypervertex set $S \subset \Hcal$, we denote its boundary set as 
\[
\partial S := \{h_1 \cup h_2 \mid h_1 \in S, h_2 \notin S, h_1 \cap h_2 = \emptyset\}  \,.
\]
Note that $\partial S$ is a set of $m$-th order hyperedges and non-edges.
\label{def:boundary_set}
\end{definition}

\begin{definition}[Hyperedge Expansion of a Hypervertex Set]
For any hypervertex set $S \subset \Hcal$, we denote the hyperedge expansion of the set $S$ as 
\[
\phi_S := \frac{\abs{\partial S}}{\abs{S}} \,.
\]
\label{def:set_expansion} 
\end{definition}

\begin{definition}[Hyperedge Expansion of a Hypergraph]
Given an $m$-uniform hypergraph $\Gcal = (\Vcal, \Ecal)$ with induced hypervertices $\Hcal$, we denote the hyperedge expansion of the hypergraph $\Gcal$ as 
\[
\phi_\Gcal := \min_{S\subset \Hcal, 0 < S \leq N/2} \phi_S = \min_{S\subset \Hcal, 0 < S \leq N/2} \frac{\abs{\partial S}}{\abs{S}} \,.
\]
We also call $\phi_\Gcal$ the Cheeger constant of the hypergraph.
\label{def:hypergraph_expansion} 
\end{definition}


\subsection{Hypergraph Laplacian}

In this section, we introduce our hypergraph Laplacian related definitions. 

\begin{definition}[$\zeta$-function]
$\zeta: \R^m \to \R$ is a function defined as
\[
\zeta(v_1,\dots,v_m) = \sum_{\Ical \subset [m], \abs{\Ical} = m/2} \left(\sum_{i\in \Ical} v_i - \sum_{j\notin \Ical} v_j\right)^m   
\,.
\]
\label{def:zeta_function}
\end{definition}

\begin{definition}[Hypergraph Laplacian]
Given an $m$-uniform hypergraph $\Gcal = (\Vcal, \Ecal)$, we use $L \in \R^{n^\om}$ to denote its Laplacian tensor, which fulfills 
\[  
\inprod[L]{v^\om} = \frac{1}{m! \binom{m}{m/2}} \sum_{(i_1,\dots,i_m)\in \Ecal} \zeta(v_{i_1},\dots,v_{i_m}) 
\,.
\]
\label{def:laplacian}
\end{definition}

\begin{definition}[Rayleigh Quotient]
For any hypergraph Laplacian $L$ and non-zero vector $v \in \R^n$, the Rayleigh quotient $R_L(v)$ is defined as 
\[
R_L(v) := \frac{\inprod[L]{v^\om}}{\norm{v}^m} \,.   
\]
\label{def:rayleigh}
\end{definition}

\begin{definition}[Signed Laplacian Tensor]
Given an $m$-uniform hypergraph $\Gcal = (\Vcal, \Ecal)$ with a sign vector $y \in \{-1,+1\}^n$, we use $L_y \in \R^{n^\om}$ to denote its signed Laplacian tensor, which fulfills
\[  
\inprod[L_y]{v^\om} = \frac{1}{m! \binom{m}{m/2}} \sum_{(i_1,\dots,i_m)\in \Ecal} \zeta(y_{i_1} v_{i_1},\dots, y_{i_m} v_{i_m}) 
\,.
\]
\label{def:signed_laplacian}
\end{definition}

\begin{remark}
The hypergraph Laplacian tensor $L$ can be viewed as a signed Laplacian tensor $L_y$, by taking $y = \onevct$. 
\end{remark}


Here we provide some important properties of hypergraph Laplacians and Rayleigh quotients.

\begin{lemma}[Laplacian Eigenpair]
For any hypergraph Laplacian $L$, $\onevct$ is an eigenvector of $L$ with a minimum eigenvalue of $0$. Similarly, for any signed hypergraph Laplacian $L_y$, $y$ is an eigenvector of $L_y$ with an eigenvalue of $0$.
\label{lemma:laplacian_eigenpair}
\end{lemma}
\begin{proof}[Proof of Lemma \ref{lemma:laplacian_eigenpair}]
This follows directly from the definition of $\zeta$-function, and Definition \ref{def:laplacian}, \ref{def:signed_laplacian}.
\end{proof}

\begin{lemma}[Invariant under Scaling]
For any non-zero $\alpha \in \R$, we have 
\[
R_L(v) = R_L(\alpha v) \,.
\]
\label{lemma:invariant_scaling}
\end{lemma}
\begin{proof}[Proof of Lemma \ref{lemma:invariant_scaling}]
\begin{align*}
R_L(\alpha v)
&= \frac{1}{m! \binom{m}{m/2}} \sum_{(i_1,\dots,i_m)\in \Ecal} \frac{\zeta(\alpha v_{i_1},\dots,\alpha v_{i_m}) }{\norm{\alpha v}^m} \\
&= \frac{1}{m! \binom{m}{m/2}} \sum_{(i_1,\dots,i_m)\in \Ecal} \frac{\alpha^m \zeta(v_{i_1},\dots,v_{i_{m}}) }{\alpha^m \norm{v}^m} \\
&= \frac{1}{m! \binom{m}{m/2}} \sum_{(i_1,\dots,i_m)\in \Ecal} \frac{\zeta(v_{i_1},\dots,v_{i_{m}}) }{\norm{v}^m} \\
&= R_L(v)
\,.
\end{align*}
\end{proof}

\begin{lemma}[Invariant under Scaling]
For any $\delta \in \R$ and $v\in \R^n$, $v\perp\onevct$, we have 
\[
R_L(v) \geq R_L(v + \delta \onevct) \,.
\]
\label{lemma:invariant_shifting}
\end{lemma}

\begin{proof}[Proof of Lemma \ref{lemma:invariant_shifting}]
Note that 
\begin{align*}
R_L(v + \delta \onevct)
&= \frac{1}{m! \binom{m}{m/2}} \sum_{(i_1,\dots,i_m)\in \Ecal} \frac{ \zeta(v_{i_1}+\delta,\dots,v_{i_{m}}+\delta) }{ \norm{v + \delta \onevct}^m} \\
&= \frac{1}{m! \binom{m}{m/2}} \sum_{(i_1,\dots,i_m)\in \Ecal} \frac{ \zeta(v_{i_1},\dots,v_{i_{m}}) }{ \norm{v + \delta \onevct}^m} \\
&\leq \frac{1}{m! \binom{m}{m/2}} \sum_{(i_1,\dots,i_m)\in \Ecal} \frac{ \zeta(v_{i_1},\dots,v_{i_{m}}) }{ \norm{v}^m} \\
&= R_L(v)
\,,
\end{align*}
where the inequality follows from the fact that 
\begin{align*}
\norm{v + \delta \onevct}^m
&= \left(\sum_i (v_i + \delta)^2\right)^{m/2} \\
&= \left(\sum_i v_i^2 + n \delta^2 + 2\delta \sum_i v_i \right)^{m/2} \\
&= \left(\sum_i v_i^2 + n\delta^2\right)^{m/2} \\
&\geq \left(\sum_i v_i^2\right)^{m/2} \\
&= \norm{v}^m 
\,.
\end{align*}
\end{proof}

\begin{lemma}[Signed Rayleigh Quotient Lower Bound]
For any hypergraph with Laplacian $L$ and signed Laplacian $L_y$, and for any $\delta \in \R$ and $v\in \R^n$, $v\perp y$, we have 
\[
R_{L_y}(v) \geq R_L(v \circ y + \delta \onevct) 
\,.
\]
\label{lemma:rq_lowerbound}
\end{lemma}

\begin{lemma}[Existance of Degree Tensor]
For any hypergraph Laplacian $L$, let $A$ denote the corresponding symmetric adjacency tensor, such that $A_{i_1,\dots,i_m} = 1$ if and only if $(i_1,\dots,i_m) \in \Ecal$. Then, there exists a high-order degree tensor $D$ fulfilling:
1) $D_{i_1,\dots,i_m} = 0$ if $i_1,\dots,i_m$ are all different, and 
2) $\inprod[D-A]{v^\om} = \inprod[L]{v^\om}$
for every vector $v \in \R^n$.
\label{lemma:exist_D_tensor}
\end{lemma}


\subsection{High-order Cheeger-type Inequality}

Here, we connect the previous two sections and provide our novel high-order Cheeger-type inequality.

\begin{theorem}[High-order Cheeger-type Inequality]
For any $m$-uniform hypergraph $\Gcal = (\Vcal, \Ecal)$ with Laplacian $L$, we have 
\[
\lsec(L) \geq \phi_\Gcal^m \,.
\]
\label{thm:cheeger_inequality}
\end{theorem}


\begin{remark}
Here we would like to discuss our choice of hypergraph Laplacian as well as the related high-order definitions.
Despite the fact that there have been multiple works trying to generalize some graph properties (including hypergraph Laplacian, Rayleigh quotient, hyperedge expansion, Cheeger constant, among others) to hypergraphs \citep{li2018submodular, mulas2021cheeger, yoshida2019cheeger,chan2018spectral, chen2017fiedler,chang2020hypergraph}, we want to highlight that none above fulfill the requirement of our high-order inference tasks. 

To obtain a set of hypergraph definitions that are useful for our inference task, we need to fulfill all the lemmas in this section.
In particular, here are the necessary conditions:
\begin{itemize}
    \item To fulfill Lemma \ref{lemma:invariant_scaling}, each term in the parenthesis in the $\zeta$-function (Definition \ref{def:zeta_function}) must have a coefficient of $1$ or $-1$.
    \item To fulfill Lemma \ref{lemma:invariant_shifting}, the $\zeta$-function (Definition \ref{def:zeta_function}) and the norm in the denominator of the Rayleigh quotient (Definition \ref{def:rayleigh}) must have the same power. Additionally, the summation of coefficients of the terms in the parenthesis in the $\zeta$-function (Definition \ref{def:zeta_function}) must be $0$. Combining with the last condition, it requires the number of plus and minus terms to be equal.
    \item To fulfill Lemma \ref{lemma:exist_D_tensor}, the number of minus terms in the parenthesis in the $\zeta$-function (Definition \ref{def:zeta_function}) must be odd, so that $\inprod[-A]{v^\om}$ can be canceled by $\inprod[L]{v^\om}$ in those entries without repeating indices.
    \item To fulfill Lemma \ref{lemma:laplacian_eigenpair}, the Rayleigh quotient must achieve minimum with $v = \onevct$.
\end{itemize} 
At this point it should be clear that we cannot simply use some arbitrary definition from the prior literature. Our hypergraph definitions are carefully constructed and chosen, to fulfill all the necessary conditions above.
\label{remark:choice_of_definitions}
\end{remark}
 
\begin{remark}
Regarding the choice of order $m$, note that if $m$ is odd, the task of label recovery will not be feasible in the current formulation, since the Rayleigh quotient will be unbounded below (instead of a minimum $0$).
From Remark \ref{remark:choice_of_definitions} it can be noticed that if we keep the current definition of $\zeta$-functions, the only possible $m$'s are $2, 6, 10, 14,\dots$ The definition of $\zeta$-functions can be generalized for other even orders, such that for every $m$-tuple $\Ical$, we take the average of many $m$-power terms by iterating through the permutation of all possible signs. For clarity of exposition we keep the current simpler definition, and focus on the case of $m = 2, 6, 10, 14,\dots$  
\label{remark:m_order}
\end{remark}

\section{Exact Recovery of True Labels}
In this section, we present an inference algorithm, which recovers the underlying true labels in our model. To do so, we take a two stage approach. In the first stage we only utilize the hyperedge information observed from $X$, and this allows us to narrow down our solution space to two possible solutions. In the second stage, with the help of the node information observed from $z$, we are able to infer the correct labeling of the nodes. 


\subsection{Stage One: Inference up to Permutation}

We start by considering the following combinatorial problem
\begin{align}
\maximize_{y} \qquad  &\inprod[X]{y^\om} \nonumber \\
\st \qquad
&y \in \{-1,+1\}^n \,. 
\label{eq:opt_combinatorial}
\end{align}
The issue with this optimization formulation is that the problem is not convex, which makes the analysis hard and intractable. 
Instead, we consider the following relaxed version
\begin{align}
\maximize_{Y} \qquad  &\inprod[X]{Y} \nonumber\\
\st \qquad
&Y \in \vpsd \nonumber \\
& -1 \leq Y_\text{odd} \leq 1 \,, \forall \text{odd} \in \sgmc \nonumber \\
&Y_\text{even} = 1 \,, \forall \text{even} \in \sgmm
\,.
\label{eq:opt_primal} 
\end{align}
Alternatively, we represent the last two constraints as $-1 \leq Y_\sgmc \leq 1$, and $Y_\sgmm = 1$.
Recall that $\vpsd$ is the convex cone of rank-one tensors. The motivation is that we are using a rank-one tensor $Y$ instead of the outer product $y^\om$, so that the problem becomes convex in the objective function and the constraints. 
We have $-1 \leq Y_\sgmc \leq 1$ because the product of $y$'s is either $-1$ or $+1$, and we have $Y_\sgmm = 1$ because if every $y_i$ repeats an even number of times, we know the product must be $+1$.

It remains to prove correctness of program \eqref{eq:opt_primal}. We are interested in identifying the regime, in which \eqref{eq:opt_primal} returns the exact rank-one tensor solution $Y^\ast := \yast^{\om} = (-\yast)^{\om}$. 
We now present our main theorem on inference. The proof can be found in Appendix \ref{appendix:proof_recovery}.

\begin{theorem}[Inference from Hyperedge Observation]
For an $m$-order structured prediction model with underlying hypergraph $\Gcal = (\Vcal, \Ecal)$ and hyperedge observation $X$, 
the rank-one tensor solution $Y^\ast := \yast^{\om} = (-\yast)^{\om}$ can be recovered from the convex optimization program \eqref{eq:opt_primal} with probability at least 
$1 - \epsilon_1(\phi_\Gcal, n, p)$,
where
\begin{align}
\epsilon_1(\phi_\Gcal, n, p)
= 
2n^m  \exp\left(-\frac{(1-2p)^2 \phi_\Gcal^{2m}}{8n^m \cdot \max(\abs{\Ecal}, n^{m-1})}\right) 
+ \frac{16(1-p) \abs{\Ecal}}{(1-2p)^2 \phi_\Gcal^{2m}}
\,.
\label{eq:epsilon1}
\end{align}
\label{thm:stage1}
\end{theorem}

\begin{remark}
A natural question to ask is, under what topological and statistical conditions can we obtain a high probability guarantee from Theorem \ref{thm:stage1}. 
An observation is that if the Cheeger constant of the underlying hypergraph is large, or the noise level is small, recovery is more likely to succeed. In Section \ref{section:example_graphs}, we analyze at some interesting classes of hypergraphs with good expansion property (large Cheeger constant), so that high probability recovery can be guaranteed.
\end{remark}


\subsection{Stage Two: Exact Inference}

In the previous section, we established the high probability inference guarantee for the rank-one tensor solution $Y^\ast := \yast^{\om} = (-\yast)^{\om}$. By taking a factorization step, we know either $\yast$ or $-\yast$ is the correct label vector. In this section, our goal is to decide the correct labeling using the node observation $z$. 

\begin{theorem}
Let $y \in \{\yast, -\yast\}$.
The correct label vector $\yast$ can be recovered from program 
\begin{align}
z^\top \yast = \max_{y \in \{\yast, -\yast\}} z^\top y \,. 
\label{opt:stage2}
\end{align}
with probability at least 
$1 - \epsilon_2(n, q)$,
where
\begin{align}
\epsilon_2(n, q)
= 
\econst^{-(1-2q)^2 n / 2}
\,.
\end{align}
\label{thm:stage2}
\end{theorem}

\begin{proof}[Proof of Theorem \ref{thm:stage2}]
Applying Hoeffding's inequality, we obtain that 
\begin{align*}
\Prob{z^\top \yast \leq -z^\top \yast} 
&= \Prob{z^\top \yast \leq 0} \\
&\leq \econst^{-(1-2q)^2 n / 2} 
\,.
\end{align*}
\end{proof}


\begin{corollary}
Combining the results in Theorem \ref{thm:stage1} and \ref{thm:stage2}, we obtain that exact inference of the correct label vector $y = \yast$ can be achieved with probability at least 
\begin{equation}
1 -  \epsilon_1(\phi_\Gcal, n, p) - \epsilon_2(n, q) \,.
\label{eq:combined_bound}
\end{equation}
\label{corollary:two_stage}
\end{corollary}

\begin{proof}[Proof of Corollary \ref{corollary:two_stage}]
Apply a union bound to Theorem \ref{thm:stage1} and \ref{thm:stage2}.
\end{proof}

\begin{remark}
Given $p, q < 0.5$, we observe that as long as $n \to \infty$, we know $\epsilon_2(n, q) \to 0$ (an exponential decay). As a result, whether one can obtain a high probability guarantee in the shape of $1 - O(n^{-1})$ depends only on the order of $\phi_\Gcal$, that is, the topological structure of the underlying hypergraph. In Section \ref{section:example_graphs}, we investigate hypergraphs with good expansion properties, that allow us to achieve a high probability guarantee of exact inference.  
\end{remark}
\section{Examples of Hypergraphs with Good Expansion Property}
\label{section:example_graphs}
In this section, we consider some example classes of hypergraphs with good expansion property, and demonstrate that they lead to high probability guarantees in our exact inference algorithm.

We first analyze complete hypergraphs.

\begin{definition}[Complete Hypergraphs]
An $m$-uniform hypergraph $\Gcal = (\Vcal, \Ecal)$ is complete, if for every $m$-vertex tuple $e = (i_1,\dots,i_m)$, we have $e \in \Ecal$.
\end{definition}

\begin{proposition}[Expansion Property of Complete Hypergraphs]
For any complete hypergraph $\Gcal = (\Vcal, \Ecal)$, we have 
\[
\phi_\Gcal
= \frac{N}{2}
\,.   
\]
\end{proposition}
\begin{proof}
Recall that $\Hcal$ is the set of induced hypervertices. 
By definition \ref{def:set_expansion}, with any $S \subset \Hcal$ with $0 < \abs{S} \leq N/2$, we have 
\[
\phi_S 
= \frac{\abs{S} \cdot \abs{\Hcal \setminus S}}{\abs{S}}   
= \abs{\Hcal \setminus S}
\,.
\]
Taking a minimum as in definition \ref{def:hypergraph_expansion}, we obtain $\phi_\Gcal = \frac{N}{2}$.
\end{proof}

\begin{corollary}[Exact Inference in Complete Hypergraphs]
Let $\Gcal = (\Vcal, \Ecal)$ be a complete hypergraph. Assume $p,q < 0.5$. Then exact inference can be achieved using the proposed two-stage approach with probability tending to $1$ with a sufficiently large $n$.
\end{corollary}

\begin{proof}
Note that for complete hypergraphs, we have 
\[
\phi_\Gcal 
= \frac{N}{2} 
= \frac{1}{2}\binom{n}{m/2} 
\geq \frac{1}{2}\left(\frac{n}{m/2}\right)^{m/2} 
= \frac{2^{m/2-1}}{m^{m/2}} \cdot n^{m/2} 
\,.
\]
Substituting $\phi_\Gcal$ into \eqref{eq:combined_bound}, as long as $n$ is greater than some constant $c_0$, we have 
\[
1 -  \epsilon_1(\phi_\Gcal, n, p) - \epsilon_2(n, q) 
\geq 
1 - O(n^{-1})    
\,.
\]
\end{proof}

Next, we focus on the case of regular expanders.

\begin{definition}[$d$-regular Expander]
An $m$-uniform hypergraph $\Gcal = (\Vcal, \Ecal)$ is a $d$-regular expander hypergraph with constant $c > 0$, if for any induced hypervertex set $S \subset \Hcal$ with $0 < \abs{S} \leq N/2$, we have 
\[
\abs{\partial S} \geq c \cdot d \cdot \abs{S} \,.    
\]

\end{definition}

\begin{proposition}[Expander Hypergraphs]
For any $d$-regular expander hypergraph $\Gcal = (\Vcal, \Ecal)$ with constant $c>0$, we have 
\[
\phi_\Gcal
= c d
\,.   
\]   
\end{proposition}
\begin{proof}
By definition \ref{def:set_expansion}, with any $S \subset \Hcal$ with $0 < \abs{S} \leq N/2$, we have 
\[
\phi_S 
\geq \frac{c \cdot d \cdot \abs{S}}{\abs{S}} 
= c d
\,.
\]
Taking a minimum as in definition \ref{def:hypergraph_expansion}, we obtain $\phi_\Gcal = c d$.
\end{proof}

\begin{corollary}[Exact Inference in Expander Hypergraphs]
Let $\Gcal = (\Vcal, \Ecal)$ be a $d$-regular expander hypergraph with constant $c>0$. Assume $p,q < 0.5$. Then exact inference can be achieved using the proposed two-stage approach with probability tending to $1$ with a sufficiently large $n$, if 
\[
d \in \Omega(n \cdot (\log n)^{1/2m}) \,.    
\]
\end{corollary}

\begin{proof}
First note that the exponential term in \eqref{eq:epsilon1} is the dominating factor. Substituting $\phi_\Gcal = cd$ into the first term of \eqref{eq:epsilon1}, we want to ensure
\[
2 n^m \exp\left(- \frac{(1-2p)^2 (cd)^{2m}}{8n^m \cdot \max(\abs{\Ecal}, n^{m-1})}\right) 
\leq 
c_0 n^{-1}
\,,
\]
for some constant $c_0 > 0$. Note that $\abs{\Ecal} \in O(n^m)$. 
As a result, a sufficient condition for $d$ is $d \in \Omega(n \cdot (\log n)^{1/2m})$.
Substituting $d$ into the other term also fulfills the $O(n^{-1})$ bound.
\end{proof}


\subsection{Simulation Results}
We test the proposed method on a hypergraph with order $m=6$; see Figure \ref{figs:simulation}. We fix the number of nodes $n = 10$. We focus on the Stage One, and check how many labels can be recovered up to permutation of the signs. 
We implement a tensor projected gradient descent solver motivated by \citet{ke2022exact,han2013unconstrained}. 
For each setting we run $20$ iterations.
Our results suggest that if the noise level $p$ is small and the hypergraph Cheeger constant $\phi_\Gcal$ is large, the proposed algorithm performs well and recovers the underlying group structure.
This matches our theoretic findings in Theorem \ref{thm:stage1}.

\begin{figure}[h!]
\centering
\includegraphics[width=0.35\linewidth]{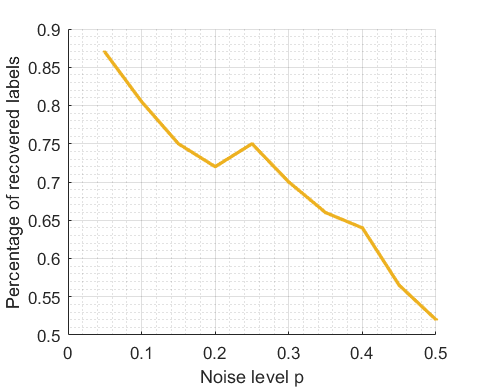}
\includegraphics[width=0.35\linewidth]{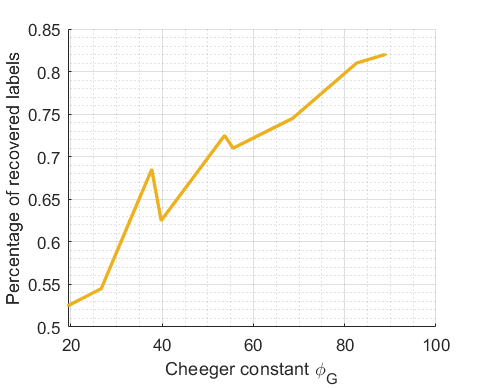}
\caption{Simulations with different noise levels $p$ (top) and Cheeger constant $\phi_G$ (down). Our results suggest that if the noise level $p$ is small and the hypergraph Cheeger constant $\phi_\Gcal$ is large, the proposed algorithm performs well and recovers the underlying group structure.}
\label{figs:simulation}
\end{figure}

\bibliography{0_main.bib}
\bibliographystyle{plainnat}

 
\newpage
\appendix

\section{Proofs of Hypergraph Structural Properties and Cheeger-type Inequality}
\label{appendix:proof_lemma}
For clarity of presentation, in the following proofs, we use 
\[
RN_L(v) := \sum_{(i_1,\dots,i_m)\in \Ecal} \zeta(v_{i_1},\dots,v_{i_m})
\]
to denote the numerator part of the Rayleigh quotient of Laplacian $L$, and 
\[
RD_L(v) := \norm{v}^m
\]
to denote the denominator part.

\begin{proof}[Proof of Lemma \ref{lemma:rq_lowerbound}]
Regarding the numerator, we have  
\begin{align*}
RN_L(v \circ y + \delta \onevct)
&= \sum_{(i_1,\dots,i_m)\in \Ecal} \left[( v_{i_1} y_{i_1} + \delta + \dots +  v_{i_{m/2}} y_{i_{m/2}} + \delta -  v_{i_{m/2+1}} y_{i_{m/2+1}} - \delta - \dots -  v_{i_m} y_{i_m} - \delta)^m + \dots \right] \\
&= \sum_{(i_1,\dots,i_m)\in \Ecal} \left[( v_{i_1} y_{i_1} + \dots +  v_{i_{m/2}} y_{i_{m/2}}  -  v_{i_{m/2+1}} y_{i_{m/2+1}}  - \dots -  v_{i_m} y_{i_m} )^m + \dots \right] \\
&= RN_{L_y}(v)
\,.
\end{align*}
Regarding the denominator, we have 
\begin{align*}
RD_L(v \circ y + \delta \onevct) 
&= \norm{v \circ y + \delta \onevct}^m \\
&= \left(\sum_i (v_i y_i + \delta)^2\right)^{m/2} \\
&= \left(\sum_i v_i^2 y_i^2 + n \delta^2 + 2\delta \sum_i v_i y_i \right)^{m/2} \\
&= \left(\sum_i v_i^2 + n\delta^2\right)^{m/2} \\
&\geq \left(\sum_i v_i^2\right)^{m/2} \\
&= \norm{v}^m \\
&= RD_{L_y}(v)
\,.
\end{align*}
\end{proof}

\begin{proof}[Proof of Lemma \ref{lemma:exist_D_tensor}]
Our goal is to construct a tensor $D$ fulfilling $\inprod[D-A]{v^\om} = \inprod[L]{v^\om}$ for any vector $v \in \R^n$.
In other words, we require 
\[
\sum_{i_1,\dots,i_m} D_{i_1,\dots,i_m} v_{i_1} \dots v_{i_m} - \sum_{i_1,\dots,i_m} A_{i_1,\dots,i_m} v_{i_1} \dots v_{i_m} = \sum_{i_1,\dots,i_m} L_{i_1,\dots,i_m} v_{i_1} \dots v_{i_m}
\,.
\]
A key observation is that on the left-hand side, $D$ and $A$ control different entries: $D_{i_1,\dots,i_m}$ is equal to $0$ in those entries without repeating indices, while $A_{i_1,\dots,i_m}$ is equal to $0$ in those entries with repeating indices. Recall that the union $\sgmm \cup \sgmc$ contains all entry indices without repeating indices.
We can rewrite the equation above as 
\[
\sum_{(i_1,\dots,i_m) \in \sgmm \cup \sgmc} D_{i_1,\dots,i_m} v_{i_1} \dots v_{i_m} - \sum_{(i_1,\dots,i_m) \notin \sgmm \cup \sgmc} A_{i_1,\dots,i_m} v_{i_1} \dots v_{i_m} = \sum_{i_1,\dots,i_m} L_{i_1,\dots,i_m} v_{i_1} \dots v_{i_m}
\,.
\]
Recall that 
\begin{align}
\inprod[L]{v^{\om}} 
&= \frac{1}{m! \binom{m}{m/2}} \sum_{(i_1,\dots,i_m)\in \Ecal} \zeta(v_{i_1},\dots,v_{i_m}) \nonumber \\
&= \frac{1}{m! \binom{m}{m/2}} \sum_{(i_1,\dots,i_m)\in \Ecal} \left[(v_{i_1} + v_{i_2} + \dots + v_{i_{m/2}} - v_{i_{m/2+1}} - \dots - v_{i_m})^m + \dots\right] 
\,.
\end{align}
The next observation is that, by expanding the $m$-power terms in the bracket, for each summand $(i_1,\dots,i_m)$ we will obtain a sequence of monomials consisting of $v_{i_1}$ through $v_{i_m}$. For example, this includes $C_0 \cdot v_{i_1} v_{i_2} \dots v_{i_{m-1}}, v_{i_m}$, $C_1 \cdot v_{i_1}^2 v_{i_2} \dots v_{i_{m-1}}$, etc., where $C$'s are coefficients. Note that all these monomials have an order of $m$.

We now analyze the monomial terms above by grouping the monomials with the same power pattern. 
Formally, we group all terms in the shape of $C \cdot v_{i_1}^{d_1} v_{i_2}^{d_2} \dots v_{i_q}^{d_q}$ together, where $d_1 \geq d_2 \geq \dots \geq d_q \geq 1$, $d_1 + \dots + d_q = m$. All permutation of subscripts are enumerated in the parenthesis.
We use $Q$ to denote the number of terms inside this group. 
Here we provide some examples of monomial groups:
\begin{itemize}
  \item $C\cdot(v_{i_1}^m + v_{i_2}^m + \dots + v_{i_m}^m)$, with pattern $d_1 = m$. In this case $Q = m$.
  \item $m = 6, C\cdot (v_{i_1}^5 v_{i_2}^1 + v_{i_1}^5 v_{i_3}^1 + \dots + v_{i_6}^5 v_{i_5}^1)$, with pattern $d_1 = 5$, $d_2 = 1$. In this case $Q = 6\cdot 5 = 30$.
  \item $m = 6, C\cdot (v_{i_1}^3 v_{i_2}^2 v_{i_3}^1 + v_{i_1}^3 v_{i_2}^2 v_{i_4}^1 + \dots + v_{i_6}^3 v_{i_5}^2 v_{i_4}^1)$, with pattern $d_1 = 3$, $d_2 = 2$, $d_3 = 1$. In this case $Q = 6\cdot 5\cdot 4 = 120$.
\end{itemize}

Next we discuss the power pattern in each group. 

If the power pattern is $d_1 = \dots = d_m = 1$ (i.e., all $1$-power components), we know the coefficient $C$ is equal to $-\frac{1}{m! \binom{m}{m/2}}$ by counting. Thus, these terms with all $1$-power components are already balanced by $- \sum_{(i_1,\dots,i_m) \notin \sgmm \cup \sgmc} A_{i_1,\dots,i_m} v_{i_1} \dots v_{i_m}$ on the left-hand side.
This also shows the necessity of introducing the factor $\frac{1}{m! \binom{m}{m/2}}$ in the definition of hypergraph Laplacians: it ensures that the term $v_{i_1} v_{i_2} \dots v_{i_m}$ in the expanded form of $\inprod[L]{v^\om}$ has a coefficient of $1$.

For groups with other power patterns ($d_1 \geq \dots d_q \geq 1$, $q < m$), we balance them by setting $D$ entries with indices in the permutation of $(i_1,\dots, i_1, i_2,\dots,i_2, \dotsb\dotsb, i_q,\dots,i_q)$, in which $i_1$ repeats $d_1$ times, $i_2$ repeats $d_2$ times, etc.
We set the value of these $D$ entries to: $CQ\cdot \sum_{i_{q+1},\dots,i_m} A_{i_1,\dots,i_m} = CQ\cdot \sum_{i_{q+1},\dots,i_m} \onefun[(i_1,\dots,i_m)\in \Ecal]$. In other words, for $D$ entries containing index $i_1$ through $i_q$, we set them to be equal to those $A$ entries containing index $\{i_1,\dots,i_m\} \setminus \{i_1,\dots,i_q\}$.
This can be illustrated using the same examples above:
\begin{itemize}
  \item Set $D_{i_1,\dots,i_1} = CQ\sum_{i_2,\dots,i_m} A_{i_1,\dots,i_m}$.
  \item Set $D_{i_1,i_1,i_1,i_1,i_1,i_2} = CQ\sum_{i_3,i_4,i_5,i_6} A_{i_1,\dots,i_m}$, and the same for all symmetric entries in $\sigma(i_1,i_1,i_1,i_1,i_1,i_2)$.
  \item Set $D_{i_1,i_1,i_1,i_2,i_2,i_3} = CQ\sum_{i_4,i_5,i_6} A_{i_1,\dots,i_m}$, and the same for all symmetric entries in $\sigma(i_1,i_1,i_1,i_1,i_1,i_2)$. 
\end{itemize}
After the procedure is done, we have $\inprod[L]{v^\om} = \inprod[D-A]{v^\om}$, and $D_{i_1,i_2,\dots,i_m} = 0$ if $i_1,i_2,\dots,i_m$ are all different.
\end{proof}






\begin{proof}[Proof of Theorem \ref{thm:cheeger_inequality}]
Suppose $v$ is the eigenvector associated with the second smallest eigenvalue $\lsec(L)$.
By Lemma \ref{lemma:invariant_scaling}, we assume $\norm{v} = 1$ without loss of generality.
We also sort the entries of $v$, i.e., $v_1 \leq \dots \leq v_n$.

Our first step is to set up a helper vector $u$ based on the second smallest eigenvector $v$.
Recall that $\Hcal$ is the set of induced hypervertices (see Definition \ref{def:hypervertices}). 
For any hypervertex $h_j = \{i_1,\dots,i_{m/2}\} \in \Hcal$, based on the second eigenvector $v$, we define the $v$-value of $h_j$ by 
\[
v(h_j) := \frac{1}{m/2} \left(v_{i_1} + \dots + v_{i_{m/2}}\right) 
\,.  
\]
This allows us to sort all the hypervertices by their $v$-value, such that the hypervertex indices fulfill
\[
v(h_1) \leq v(h_2) \leq \dots \leq v(h_N) 
\,,
\]
and recall that $N := \abs{\Hcal}$.  
We use $M$ to denote the smallest integer fulfilling 
$M \geq \frac{1}{2}\abs{\Hcal}$,
and we introduce a shifting operation by defining 
$u := v - v(h_M) \cdot \onevct$.
We further find a constant $c_u > 0$, and we scale $u$ by multiplying $u$ with $c_u$ so that it fulfills 
$(u_{1} + \dots + u_{{m/2}})^2 + (u_{n-m/2+1} + \dots + u_{n})^2 
= 
1$.
Note that $u$ fulfills the following properties:
\begin{itemize}
  \item $(u_{1} + \dots + u_{{m/2}})^2 + (u_{n-m/2+1} + \dots + u_{n})^2 = 1$.
  \item $u_1 \leq \dots \leq u_n$.
  \item $u_M = 0$.
  \item $R_L(v) \geq R_L(u)$, by Lemma \ref{lemma:invariant_scaling} and \ref{lemma:invariant_shifting}.
\end{itemize}

The second step of our proof is to construct a random set of hypervertices $S_t$.
To do so, we define $t$ to be a random variable on the support $[u_{1} + \dots + u_{{m/2}}, u_{n-m/2+1} + \dots + u_{n}]$,
with probability density function $f(t) = 2\abs{t}$.
It can be verified that $t$ is a valid random variable, because 
\[
\int_{u_{1} + \dots + u_{{m/2}}}^{u_{n-m/2+1} + \dots + u_{n}} 2\abs{t} 
= (u_{1} + \dots + u_{{m/2}})^2 + (u_{n-m/2+1} + \dots + u_{n})^2
= 1
\,.
\]
Based on $t$, we can construct a random set of hypervertices $S_t$ as follows 
\[
S_t := \lbrace h_j \mid h_j = \{i_1,\dots,i_{m/2}\}, u_{i_1} + \dots + u_{i_{m/2}} \leq t \rbrace \,.  
\]
Here we consider the size of $S_t$ in the average case. We have 
\[
\Expect[t]{\abs{S_t}}  
= \sum_j \Prob[t]{u_{h_j} \leq t}
\,,
\]
and 
\[
\Expect[t]{\abs{\Hcal \setminus S_t}}  
= \sum_j \Prob[t]{u_{h_j} > t}
\,.
\]
Combining the two leads to
\begin{align*}
\Expect[t]{\min(\abs{S_t}, \abs{\Hcal\setminus S_t})}
&= \sum_j \Prob[t]{u_{h_j} \leq t < 0} + \sum_j \Prob[t]{u_{h_j} > t \geq 0 } \\
&= \sum_j u_{h_j}^2
\,.
\end{align*}
We also consider the boundary set $\partial S_t$. By Definition \ref{def:boundary_set}, a hyperedge $e = h_1 \cup h_2$ belongs to $\partial S_t$, if $h_1 \in S_t$, $h_2\notin S_t$, and $h_1\cap h_2 = \emptyset$.
Define the shorthand notation $u_h := \sum_{i\in h} u_i$, and assume $u_{h_1} \leq u_{h_2}$.
Then 
\[
\Prob{e\in \partial S_t} 
=
\Prob{u_{h_1} \leq t \leq u_{h_2}}
\leq 
\abs{u_{h_1} - u_{h_2}} (\abs{u_{h_1}} + \abs{u_{h_2}})
\,.
\]

Finally, we analyze the expectation of $\abs{\partial S_t}$. It follows that 
\begin{align*}
\Expect[t]{\abs{\partial S_t}}
&= \sum_{e\in \Ecal, e = h_1 \cup h_2} \Expect[t]{\onefun[e\in \partial S_t]} \\
&= \sum_{e\in \Ecal, e = h_1 \cup h_2} \Prob[t]{e\in \partial S_t} \\
&\leq \sum_{e\in \Ecal, e = h_1 \cup h_2} \abs{u_{h_1} - u_{h_2}} (\abs{u_{h_1}} + \abs{u_{h_2}}) \\
&\overset{\text{(a)}}{\leq} \left(\sum_{e\in \Ecal, e = h_1 \cup h_2} (u_{h_1} - u_{h_2})^m \right)^{\frac{1}{m}} \left(\sum_{e\in \Ecal, e = h_1 \cup h_2} (\abs{u_{h_1}} + \abs{u_{h_2}})^{\frac{m}{m-1}}\right)^{\frac{m-1}{m}} \\
&= RN^{\frac{1}{m}} \cdot \left(\sum_{e\in \Ecal, e = h_1 \cup h_2} (\abs{u_{h_1}} + \abs{u_{h_2}})^{\frac{m}{m-1}}\right)^{\frac{m-1}{m}} \\
&= R_L(u)^{\frac{1}{m}} \cdot \norm{u} \cdot \left(\sum_{e\in \Ecal, e = h_1 \cup h_2} (\abs{u_{h_1}} + \abs{u_{h_2}})^{\frac{m}{m-1}}\right)^{\frac{m-1}{m}} \\
&\leq R_L(u)^{\frac{1}{m}} \cdot \norm{u} \cdot \left(\sum_{e\in \Ecal, e = h_1 \cup h_2} (\abs{u_{h_1}} + \abs{u_{h_2}})^{2}\right)^{\frac{m-1}{m}} \\
&\leq R_L(u)^{\frac{1}{m}} \cdot \Expect[t]{\min(\abs{S_t}, \abs{\Hcal\setminus S_t})}^{\frac{m-1}{m}}  \\
&\leq R_L(u)^{\frac{1}{m}} \cdot \Expect[t]{\min(\abs{S_t}, \abs{\Hcal\setminus S_t})} 
\,,
\end{align*}
where (a) follows from Holder's inequality.
Rearranging the terms above leads to 
\[
\Expect[t]{R_L(u)^{\frac{1}{m}} \cdot \min(\abs{S_t}, \abs{\Hcal\setminus S_t}) - \abs{\partial S_t}} \geq 0 \,.  
\]
Thus there exists some $t$ fulfilling 
\[
R_L(u)^{\frac{1}{m}} \cdot \min(\abs{S_t}, \abs{\Hcal\setminus S_t}) - \abs{\partial S_t}
\geq 
0
\,,  
\]
or equivalently,
\[
R_L(u) \geq \left(\frac{\abs{\partial S_t}}{\min(\abs{S_t}, \abs{V\setminus S_t})} \right)^m
\,.
\]
Plugging in $R_L(v) \geq R_L(u)$ on the left-hand side, and Definition \ref{def:hypergraph_expansion} on the right-hand side, we obtain 
\[
\lsec(L) 
= R_L(v)
\geq R_L(u) \geq \left(\frac{\abs{\partial S_t}}{\min(\abs{S_t}, \abs{V\setminus S_t})} \right)^m
\geq \phi_\Gcal^m 
\,.
\]
\end{proof}
\section{Proof of Exact Recovery of True Labels}
\label{appendix:proof_recovery}
\begin{proof}[Proof of Theorem \ref{thm:stage1}]
Our goal is to recover the true labels $Y^\ast := (\pm \yast)^{\om}$ (up to flipping signs) using \eqref{eq:opt_primal}. For readers' convenience, here we restate the formulation:
\begin{align*}
\maximize_{Y} \qquad  &\inprod[X]{Y} \nonumber\\
\st \qquad
&Y \in \vpsd \nonumber \\
& -1 \leq Y_\sgmc \leq 1 \nonumber \\
&Y_\sgmm = 1 \,.
\end{align*}
The Lagrangian dual problem of \eqref{eq:opt_primal} is
\begin{align*}
\minimize_{V, V^+, V^-, A} \qquad & \sum_{(i_1,\dots,i_m) \in \sgmm} V_{i_1,\dots,i_m}  + \sum_{(i_1,\dots,i_m) \in \sgmc} (V_{i_1,\dots,i_m}^+ + V_{i_1,\dots,i_m}^-) \\
\st \qquad
& V_{i_1,\dots,i_m} = 0 \,, \text{ if $i_1,\dots,i_m$ are all different}   \\
& V_{\sgmc} = V_{\sgmc}^+ - V_{\sgmc}^-  \\
& V_{\sgmc}^+ \geq 0 \\
& V_{\sgmc}^- \geq 0  \\
& -X-A + \sum_{(i_1,\dots,i_m) \in \sgmm} V_{i_1,\dots,i_m} \cdot E_{i_1,\dots,i_m} + \sum_{(i_1,\dots,i_m) \in \sgmc} (V_{i_1,\dots,i_m}^+ - V_{i_1,\dots,i_m}^-) \cdot E_{i_1,\dots,i_m} = 0\\
& A \in \cpsd 
\,,
\end{align*}
where $E_{i_1,\dots,i_m}$ is a tensor with $1$ in entry $(i_1,\dots,i_m)$, and $0$ everywhere else.
From the primal and dual problems, we obtain the following KKT conditions:
\begin{align*}
V - X  - A &= 0  
\tag{Stationarity}  \\
Y_{\sgmm} &= 1 \\
Y_{\sgmc} &\leq 1 \\
-Y_{\sgmc} &\leq 1 \\
Y &\in \vpsd 
\tag{Primal Feasibility}  \\
V_{i_1,\dots,i_m} &= 0 \\
V_{\sgmc} &= V_{\sgmc}^+ - V_{\sgmc}^- \\
V_{\sgmc}^+ &\geq 0 \\
V_{\sgmc}^- &\geq 0 \\ 
A &\in \cpsd 
\tag{Dual Feasibility}  \\
V_{\sgmc}^+ (Y_{\sgmc}-1) &= 0 \\
V_{\sgmc}^- (Y_{\sgmc}+1) &= 0 \\
\inprod[A]{Y} &= 0
\,.
\tag{Complementary Slackness} 
\end{align*}

Here we construct primal and dual variables to fulfill all KKT conditions above. 
For the primal variable we set $Y = Y^\ast$. 
For the dual variable we define $V = D_{\yast}$, where $D_{\yast}$ is the high-order degree tensor constructed from the signed Laplacian $L_\yast$, with the procedure defined in Lemma \ref{lemma:exist_D_tensor}. 
We then define $V_{\sgmm} = (D_{\yast})_\sgmm$, $V_{\sgmc}^+ = \max((D_{\yast})_\sgmc, 0)$, $V_{\sgmc}^- = \min((D_{\yast})_\sgmc, 0)$.
From the stationarity condition, we set $A = L_\yast = D_\yast - X$. 

At this point, our construction of primal and dual variables have fulfilled every KKT conditions above except the positive semidefinite condition $A = L_\yast \in \cpsd$. 
From Lemma \ref{lemma:laplacian_eigenpair}, we know $\yast$ is an eigenvector of $L_\yast$ with an eigenvalue of $0$, or equivalently, $\inprod[L_\yast]{Y^\ast} = 0$. It remains to ensure for all orthogonal vectors, we have 
$
\min_{u\perp \yast} R_{L_\yast} (u)
\geq 
0$.  
On top of that, we want to ensure the solution $Y = Y^\ast$ is unique. This further requires that 
\[
\min_{u\perp \yast} R_{L_\yast} (u)
=
\min_{u\perp \yast} \frac{\inprod[D_\yast - X]{u^\om}}{\norm{u}^m} 
>
0
\,.
\]

Without loss of generality, we fix $\norm{u} = 1$ in the following discussion. We split the terms above into
\begin{align}
\min_{u\perp \yast, \norm{u} = 1} \inprod[D_\yast - X]{u^\om}
&\geq \min_{u\perp \yast, \norm{u} = 1} \inprod[D_\yast - \Expect{D_\yast}]{u^\om} \label{eq:D_concentrate}\\
&\quad + \min_{u\perp \yast, \norm{u} = 1} \inprod[\Expect{X} - X]{u^\om} \label{eq:X_concentrate}\\
&\quad + \min_{u\perp \yast, \norm{u} = 1} \inprod[\Expect{D_\yast - X}]{u^\om} \label{eq:expectation_concentrate}
\,.
\end{align}

First we bound the expectation term \eqref{eq:expectation_concentrate}.
Note that 
\begin{align*}
\min_{u\perp \yast, \norm{u} = 1} \inprod[\Expect{D_\yast - X}]{u^\om}
&= \min_{u\perp \yast, \norm{u} = 1} \inprod[\Expect{L_\yast}]{u^\om} \\
&= (1-2p) \cdot \min_{u\perp \yast, \norm{u} = 1} \sum_{(i_1,\dots,i_m)\in \Ecal} \zeta(u_{i_1} y_{i_1}^\ast,\dots,u_{i_m} y_{i_m}^\ast) \\
&= (1-2p) \cdot R_{L_\yast} (u) \\
&\geq (1-2p) \cdot R_{L} (u \circ \yast + \delta \onevct) \\
&\geq (1-2p) \cdot \phi_\Gcal^m \,,
\end{align*}
where the first inequality follows from Lemma \ref{lemma:rq_lowerbound}, and the second inequality follows from Theorem \ref{thm:cheeger_inequality}.

Next we bound \eqref{eq:X_concentrate} using concentration inequalities.
We have 
\begin{align*}
\Prob{- \lmax(\Expect{X} - X) \leq -t}
&\leq \Prob{- \normf{\Expect{X} - X}^2 \leq -t^2}  \\
&\leq \frac{1}{t^2} \cdot \Expect{\normf{\Expect{X} - X}^2} \\
&= \frac{1}{t^2} \cdot 4p(1-p)\abs{\Ecal}
\,.
\end{align*}
Setting $t = \frac{1-2p}{2} \phi_\Gcal^m$ leads to 
\begin{align}
\Prob{- \lmax(\Expect{X} - X) \leq - \frac{1-2p}{2} \phi_\Gcal^m}
&\leq \frac{16p(1-p)\abs{\Ecal}}{(1-2p)^2 \phi_\Gcal^{2m}} 
\,.
\label{eq:X_rate}
\end{align}

Finally we bound \eqref{eq:D_concentrate}.
Using Cauchy-Schwarz inequality, we obtain 
\begin{align*}
- \abs{\inprod[\Expect{D_\yast} - D_\yast]{u^\om}} 
&\geq - \norm{\vect{\Expect{D_\yast} - D_\yast}}  \cdot \norm{\vect{u^\om}} \\
&= - \norm{\vect{\Expect{D_\yast} - D_\yast}} \cdot \norm{u}^m\\
&= - \norm{\vect{\Expect{D_\yast} - D_\yast}} \\
&\geq - n^{m/2} \cdot \norminf{\vect{\Expect{D_\yast} - D_\yast}} 
\,,
\end{align*}
where $\vect{\cdot}$ is the vectorization operator.
We then use Hoeffding's inequality for every entry. By the construction procedure defined in the proof of Lemma \ref{lemma:exist_D_tensor}, every entry of $D_\yast$ is the summation of at most $\max(\abs{\Ecal},n^{m-1})$ Rademacher random variables. We obtain
\begin{align*}
\Prob{{\vect{\Expect{D_{\yast,i_1,\dots,i_m}} - D_{\yast,i_1,\dots,i_m}}} \geq t}
&\leq 2\exp\left(-\frac{2t^2}{2^2 \max(\abs{\Ecal},n^{m-1})}\right) \\
&\leq 2\exp\left(-\frac{t^2}{2\max(\abs{\Ecal},n^{m-1})}\right) 
\,.
\end{align*}
By a union bound, we obtain
\begin{align*}
\Prob{\norminf{\vect{\Expect{D_\yast} - D_\yast}} \geq t}
&\leq 2\exp\left(-\frac{2t^2}{2^2 \max(\abs{\Ecal},n^{m-1})}\right) \\
&\leq 2n^m \exp\left(-\frac{t^2}{2\max(\abs{\Ecal},n^{m-1})}\right) 
\,.
\end{align*}
Setting $t = \frac{(1-2p)\phi_\Gcal^m}{2n^{m/2}}$  leads to 
\begin{align}
\Prob{- \abs{\inprod[\Expect{D_\yast} - D_\yast]{u^\om}}  \leq - \frac{1-2p}{2} \phi_\Gcal^m}
&\leq 2n^m \exp\left(-\frac{(1-2p)^2 \phi_\Gcal^{2m}}{8n^m \max(\abs{\Ecal},n^{m-1})}\right) 
\,.
\label{eq:D_rate}
\end{align}

Combining the results of \eqref{eq:X_rate} and \eqref{eq:D_rate} with a union bound completes the proof.
\end{proof}


\end{document}